\documentclass[letterpaper]{article} 
\usepackage{aaai2026}  
\usepackage{times}  
\usepackage{helvet}  
\usepackage{courier}  
\usepackage[hyphens]{url}  
\usepackage{graphicx} 
\usepackage{natbib}  
\usepackage{caption} 
\usepackage{float}
\usepackage{mathtools}
\usepackage[vlined,linesnumbered,ruled,titlenotnumbered]
{algorithm2e}
\usepackage{xcolor}

\newcommand{\ones}[1]{|#1|_1} 
\newcommand{\zeros}[1]{|#1|_0}

\newcommand{\OMM}{\textsc{OMM}\xspace}                             

\newcommand{\mOMM}{\ensuremath{m}\text{-}\textsc{OMM}\xspace}

\newcommand{\expect}[1]{\mathrm{E}\left[#1\right]}        
\newcommand{\E}[1]{\mathrm{E}\left[#1\right]}        

\newcommand{\rp}{\mathrm{rp}}
\newcommand{\refer}{\mathcal{R}_p}

\newcommand{\N}{\ensuremath{\mathbb{N}}} 

\newcommand{\nsga}{NSGA\nobreakdash-II\xspace}
\newcommand{\nsgaIII}{NSGA\nobreakdash-III\xspace}

\usepackage{newfloat}
\usepackage{listings}
\usepackage{amsmath,amssymb,amsthm}

\usepackage{amsmath,dsfont,amsthm}

\newtheorem{theorem}             {Theorem}
\newtheorem{lemma}      [theorem]{Lemma}

\newtheorem{definition} [theorem]{Definition}

\DeclareCaptionStyle{ruled}{labelfont=normalfont,labelsep=colon,strut=off} 
\title{Towards a Rigorous Understanding of the Population Dynamics of the NSGA-III: Tight Runtime Bounds}
\author{
    Written by AAAI Press Staff\textsuperscript{\rm 1}\thanks{With help from the AAAI Publications Committee.}\\
    AAAI Style Contributions by Pater Patel Schneider,
    Sunil Issar,\\
    J. Scott Penberthy,
    George Ferguson,
    Hans Guesgen,
    Francisco Cruz\equalcontrib,
    Marc Pujol-Gonzalez\equalcontrib
}

\author{Andre~Opris}
\affiliations{Chair of Algorithms for Intelligent Systems, University of Passau, Passau, Germany\\
    andre.opris@uni-passau.de}

\usepackage{bibentry}

\begin{document}

\maketitle

\begin{abstract}
    Evolutionary algorithms are widely used for solving multi-objective optimization problems. A prominent example is NSGA-III, which is particularly well suited for solving problems involving more than three objectives, distinguishing it from the classical NSGA-II. Despite its empirical success, the theoretical understanding of NSGA III remains very limited, especially with respect to runtime analysis. A central open problem concerns its population dynamics, which involve controlling the maximum number of individuals sharing the same fitness value during the exploration process. In this paper, we make a significant step towards such an understanding by proving tight runtime bounds for NSGA-III on the bi-objective OneMinMax ($2$-OMM) problem. Firstly, we prove that NSGA-III requires $\Omega(n^2 \log(n) / \mu)$ generations in expectation to optimize $2$-OMM assuming the population size $\mu$ satisfies $n+1 \leq \mu =O(\log(n)^c(n+1))$ where $n$ denotes the problem size and $c<1$ is a constant. Apart from~\cite{opris2025multimodal}, this is the first proven lower runtime bound for NSGA-III on a classical benchmark problem. Complementing this, we secondly improve the best known upper bound of NSGA-III on the $m$-objective OneMinMax problem ($m$-OMM) of $O(n \log(n))$ generations by a factor of $\mu /(2n/m + 1)^{m/2}$ for a constant number $m$ of objectives and population size $(2n/m + 1)^{m/2} \leq \mu \in O(\sqrt{\log(n)} (2n/m + 1)^{m/2})$. This yields tight runtime bounds in the case $m = 2$, and the surprising result that NSGA-III beats NSGA-II by a factor of $\mu/n$ in the expected runtime.
\end{abstract}

\section{Introduction}
Decision making is a fundamental aspect of many areas in artificial intelligence, where it is often important to explore trade-offs and compromises between different options before reaching a conclusion~\cite{LUUKKONEN2023102537}. Such situations are often formulated as multi-objective optimization problems, which are typically tackled using evolutionary multi-objective algorithms~\cite{STEWART2021103830}. They apply principles of nature to optimize functions with conflicting objectives, aiming to find a diverse Pareto-optimal set of solutions. This offers decision makers a range of trade-off solutions, enabling them to select the one that best aligns with their preferences~\cite{TAMSSAOUET202287}. It is therefore not surprising that such algorithms have become essential tools, widely applied across various practical domains. These include artificial intelligence~\cite{ArtIntKoziel}, often in combination with bioinformatics~\cite{HANDLBIO,MultiBioBook}, as well as constraint optimization~\cite{GARCIA2021100983}, machine learning~\cite{9515233, EMOAsMachine}, and engineering~\cite{MultiobjEngineerring, EMOAEngineer}. 
In particular, many of such real-world applications involve optimization problems with many objectives. However, a huge challenge is that, as the number of objectives increases, the Pareto front expands exponentially, making the problems increasingly complex. Additionally, identifying dependencies between individual objectives becomes more difficult. There are already differences between two and more objectives. In the case of two objectives, sorting non-dominated individuals according to the first objective naturally leads to a reverse sorting with respect to the second, making the \emph{crowding distance}, which measures the proximity of search points based on their sorting across objectives, a reliable indicator of their relative closeness. However, this relationship breaks down for problems with three or more objectives, as a solution can have a crowding distance of zero even when it is not close to other solutions (see, for example~\cite{Zheng2023Inefficiency}). As a result, \nsga~\cite{Deb2002}, the most cited EMOA ($\sim$ 56000 citations) which uses the crowding distance as a tie-breaker, succeeds in solving bi-objective problems (see for example~\cite{ZhengLuiDoerrAAAI22} for a rigorous analysis or~\cite{Deb2002} for empirical results), but fails in optimizing many problems where the number of objectives is large (compare with~\cite{Zheng2023Inefficiency} for large differences already between two and three objectives or for empirical studies~\cite{NSGAIIINEFF2003}). 
To overcome this problem, \citet{DebJain2014} designed the \nsgaIII algorithm. It uses a set of predefined reference points instead of the crowding distance. A major advantage is that these reference points can be predefined by users based on their specific needs. Hence, this algorithm has a huge practical impact ($\sim$6000 citations) and it is empirically shown that it can efficiently solve problems with at least four objectives~\cite{DebJain2014,NSGAIIIAppl,NSGAIIIAPPLL}. However, theoretical understanding of its success lags far behind its practical impact and the first papers addressing rigorous runtime analyses of this algorithm appeared only recently (see for example~\cite{WiethegerD23,OprisNSGAIII} for breakthroughs). Surprisingly, even in simple settings,~\citet{opris2025multimodal} showed that NSGA-III exhibits population dynamics that differ significantly from those of NSGA-II: NSGA-III successively iterates through all reference points, always choosing a point associated to a reference point with the fewest chosen individuals so far for the next generation, while NSGA-II treats all points with zero crowding distance equally. Hence, NSGA-III tends to spread solutions very evenly across the Pareto front (compare also with~\cite{CHAUDHARI20221509} for empirical results). Indeed, for appropriate population sizes, it was shown in~\cite{opris2025multimodal} that NSGA-III  
outperforms NSGA-II on the pseudo-Boolean bi-objective multimodal function OJZJ for appropriate population sizes. However, they showed how NSGA-III spreads solutions evenly across the Pareto front \emph{after} converging to local optima. How this distribution evolves during exploration, particularly before reaching a local optimum or when no local optima exist at all, as in $m$-OMM, remains unclear. It is still unknown \emph{when} and \emph{why} NSGA-III performs well, or how quickly it spreads solutions across the Pareto front. As a first step toward understanding its limitations on many-objective problems, we focus on the bi-objective case, which already exhibits complex population dynamics.

\textbf{Our contribution:} We significantly increase the understanding of the population dynamics of the NSGA-III on the pseudo-Boolean $2$-OMM by investigating the \emph{maximum cover number} $\beta$, defined as the maximum number of individuals in the population sharing the same fitness vector. \citet{opris2025multimodal} showed that $\beta$ is non-increasing. Our first two results are about the time to firstly cover a subset $\mathcal{A}$ of the Pareto front of a given cardinality $\alpha$ (Lemma~\ref{lem:cover-number-general-1}), and then spread all solutions evenly on that set (Lemma~\ref{lem:cover-number-general-2}). With high probability, this time is $O(\alpha)$. On the other hand, for a given maximum cover number $\beta$, we analyze the population’s exploration towards the extreme point $1^n$. Specifically, for two constants $0 < a < b \leq 3/4$, we provide a lower bound on the time required to reduce the population’s distance to $1^n$ from at least $n^b$ to at most $n^a$. With high probability, this time is $\Omega(n \ln n / \beta)$ (Lemmas~\ref{lem:spread-of-solutions-1} and~\ref{lem:spread-of-solutions-2}). This bound increases asymptotically with $1/\beta$, which is unsurprising since a smaller $\beta$ reduces the probability of selecting individuals already close to $1^n$, and further decreasing their distance through mutation. Then, this bound can be used, in conjunction with Lemmas~\ref{lem:cover-number-general-1} and~\ref{lem:cover-number-general-2}, to further reduce the cover number by covering and afterwards spreading on a set of cardinality $\Omega(n \ln(n)/\beta)$ for carefully chosen population sizes. This again results, by Lemmas~\ref{lem:spread-of-solutions-1} and~\ref{lem:spread-of-solutions-2}, in a larger lower bound of $\Omega(n \ln n / \beta)$. By repeating this process a finite number of times, we finally reduce $\beta$ to a value of $O(\mu/n)$ before creating a search point $x$ with distance at most $n^{1/16}$ to $1^n$. Notably, this is the smallest possible value for $\beta$ up to a multiplicative constant, and leads to a lower bound of $\Omega(n^2 \ln(n)/\mu)$ expected generations for NSGA-III to optimize $2$-OMM (Theorem~\ref{thm:lower-bound}). To the best of our knowledge, this is the first established lower runtime bound of \nsgaIII on a classical benchmark problem without multimodality. Finally, we improve the upper bound from~\cite{OprisNSGAIII} for any constant number $m$ of objectives and $(2n/m + 1)^{m/2} \leq \mu \in O(\sqrt{\ln(n)}(2n/m + 1)^{m/2})$ by a factor of $O((2n/m + 1)^{m/2}/\mu)$, showing that $m$-OMM can be optimized in $O\left(n \ln(n) (2n/m + 1)^{m/2} / \mu\right)$ generations in expectation (Theorem~\ref{thm:Runtime-Analysis-NSGA-III-mOMM}). This aligns with the earlier lower bound for the case $m = 2$ and reveals the somewhat surprising result that NSGA-III outperforms NSGA-II by a factor of $\mu/n$ (see~\cite{DoerrQu2023a} for a lower runtime bound of $\Omega(n \ln(n))$ generations for NSGA-II to optimize $2$-OMM, if $4(n+1) \leq \mu \leq o(n^\nu)(n+1)$ for $\nu<1$). Despite the latter being the state of the art algorithm for two objectives and widely used in practice (with around 60,000 citations).


\textbf{Related work:}
The mathematical runtime analysis of modern practical MOEAs began only recently. \citet{ZhengLuiDoerrAAAI22} conducted the first runtime analysis of NSGA-II on classical benchmark functions which was the starting point for plenty of successive works in similar contexts dealing with optimizing bi-objective functions by NSGA-II~\cite{Qu2022PPSN, DaOp2023, NSgaIIBeat, Dang2024, DoerrQ23b, DoerrApprox, UpBian, 2025Lessons} which has even been extended on combinatorial optimization problems like minimum spanning trees or subset selection~\cite{NSGAIICombIJCAI, MOEASubset}. Very recently, variants of the NSGA-II were proposed to overcome its shortcomings in solving many-objective problems by adding a simple tie-breaking rule~\cite{Krejca2025b} or by using an alternative version of the crowding distance~\cite{ZhengDoerrCrowding}. The most prominent algorithms when dealing with many-objectives are the SPEA2, SMS-EMOA and NSGA-III which have also been analyzed successfully~\cite{WiethegerD23,Zheng_Doerr_2024,OprisNSGAIII,DoerrNearTight,Opris2025,opris2025multimodal}. 
However, up to~\cite{opris2025multimodal,DoerrQu2023a} there were no proven tight runtime bounds on the performance of NSGA-II or NSGA-III on classical benchmark functions, despite investigating limitations of EMOAs and proving lower bounds by analyzing their population dynamics is a highly active area of research. For example,~\citet{Opris2025PAES} abalyzed the PAES-25 evolutionary strategy with one-bit mutation on the $m$-objective \textsc{LeadingOnesTrailingZeros} problem, proving tight runtime bounds of $\Theta(n^3)$ for $m = 2$, $\Theta(n^3 \ln n)$ for $m = 4$, and $\Theta\left(n(2n/m)^{m/2} \ln(n/m)\right)$ for $m > 4$. Additional tight runtime bounds for the GSEMO algorithm on the bi-objective COCZ and OMM benchmarks are provided in~\cite{doerr2025tightruntimeGSEMO}. To the best of our knowledge, there are no existing results on the population dynamics of NSGA-III, apart from the investigations on OJZJ in~\cite{opris2025multimodal}.

\section{Preliminaries}

Given two random variables $X$ and $Y$ on $\mathbb{N}_0$, we say that $Y$ stochastically dominates $X$ if $\Pr(Y \leq c) \leq \Pr(X \leq c)$ for all $c \geq 0$. The number of ones in a bit string $x$ is denoted by $\ones{x}$ and the number of zeros by $\zeros{x}$, respectively. For any finite set $A$, we write $|A|$ to denote its cardinality. For $n \in \mathbb{N}$ let $[n]:=\{1,\ldots , n\}$, denote by $\ln$ the natural logarithm (i.e. to base $e$) and let $\text{poly}(n)$ be a placeholder for some polynomial in $n$. 

This paper is about many-objective optimization, specifically the maximization of a discrete $m$-objective function $f(x) := (f_1(x), \ldots ,f_m(x))$ for $m \in \mathbb{N}$ where each $f_i : \{0,1\}^n \to \mathbb{N}_0$ for $i \in \{1, \ldots , m\}$. When $m=2$, the function is also called \emph{bi-objective}. For a bit string $x$ let $x:=(x^1, \ldots , x^{m/2})$ where all $x^j$ are of equal length $2n/m$. For a subset $N \subseteq \{0,1\}^n$, we define $f(N) := \{f(x) \mid x \in N\}$. Given two search points $x,y \in \{0,1\}^n$, $x$ \emph{weakly dominates} $y$, denoted by $x \succeq y$, if $f_i(x) \geq f_i(y)$ for all $i \in [m]$ and $x$ \emph{(strictly) dominates} $y$, denoted by $x \succ y$, if one inequality is strict; if neither $x \succeq y$ nor $y \succeq x$ then $x$ and $y$ are \emph{incomparable}. A set $S \subseteq \{0,1\}^n$ is a \emph{set of mutually incomparable solutions} with respect to $f$ if all search points in $S$ are incomparable. 
Each solution not dominated by any other in $\{0, 1\}^n$ is called \emph{Pareto-optimal}. A mutually incomparable set of these solutions that covers all possible non-dominated fitness values is called a \emph{Pareto(-optimal) set} of $f$. For a population $P_t \subset \{0,1\}^n$ and $v \in \mathbb{N}_0^m$ denote by $c_t(v):=|\{x \in P_t \mid f(x) = v\}|$ the \emph{cover number} of $v$, i.e. the number of individuals from $P_t$ with fitness vector $v$. We say that $v$ is \emph{covered} if $c_t(v) \geq 1$.

\begin{algorithm}[t]
	Initialize $P_0 \sim \text{Unif}( (\{0,1\}^n)^{\mu})$\\
	\For{$t:= 0$ to $\infty$}{
		Initialize $Q_t:=\emptyset$\\
		\For{$i=1$ to $\mu$}{
                Sample $s$ from $P_t$ uniformly at random\\
                Create $r$ by standard bit mutation on $s$ with mutation probability $1/n$\\
                Update $Q_t:=Q_t \cup \{r\}$\\
                }
            Set $R_t := P_t \cup Q_t$\\
		Partition $R_t$ into layers $F^1_t,F^2_t,\ldots ,F^k_t$ of non-dominated solutions \label{line:non-dominated}\\
            Find $i^* \geq 1$ such that $\sum_{i=1}^{i^*-1} \lvert{F_t^i}\rvert < \mu$ and $\sum_{i=1}^{i^*} \lvert{F_t^i}\rvert \geq \mu$\\
            Compute $Y_t = \bigcup_{i=1}^{i^*-1} F_t^i$\\
            Choose $\tilde{F}_t^{i^*} \subset F_t^{i^*}$ such that $\lvert{Y_t \cup \tilde{F}_t^{i^*}}\rvert = \mu$ with Algorithm~\ref{alg:Survival-Selection}\\
		Create the next population $P_{t+1} := Y_t \cup \tilde{F}^{i^*}_t$ \label{line:survival}\\
	}
	\caption{NSGA-III~(\protect\cite{DebJain2014}) with population size $\mu$ on an $m$-objective function $f$}
	\label{alg:nsga-iii}
\end{algorithm}
\begin{algorithm}[t]
        Compute the normalisation $f^n$ of $f$\\ 
        Associate each $x\in Y_t \cup F_t^{i^*}$ with its reference point $\rp(x)$ such that the distance between $f^n(x)$ and the line through the origin and $\rp(x)$ is minimized\\ 
        For each $r \in \refer$, set $\rho_r:=|\{x\in Y_t \mid \mathrm{rp}(x)=r\}|$\\
        Initialize $\tilde{F}_t^{i^*}=\emptyset$ and $R':=\refer$\\
        \While{true}{
        Determine $r_{\min} \in R'$ such that $\rho_{r_{\min}}$ is minimal (where ties are broken randomly)~\label{line:minimal}\\
        Determine $x_{r_{\min}} \in F_t^{i^*} \setminus \tilde{F}_t^{i^*}$ which is associated with $r_{\min}$ and minimizes the distance between the vectors $f^n(x_{r_{\min}})$ and $r_{\min}$ (where ties are broken randomly)\label{line:association}\\
        \If{$x_{r_{\min}}$ exists}{
        $\tilde{F}_t^{i^*} = \tilde{F}_t^{i^*} \cup \{x_{r_{\min}}\}$\\
        $\rho_{r_{\min}} = \rho_{r_{\min}} + 1$\\
            \If{$\lvert{Y_t}\rvert + \lvert{\tilde{F}_t^{i^*}}\rvert = \mu$}
                {\Return{$\tilde{F}_t^{i^*}$}}
                }
            \lElse{$R'=R' \setminus \{r_{\min}\}$}
        }
	\caption{Selection procedure utilizing a set $\refer$ of reference points to maximize an $m$-objective function $f$}
	\label{alg:Survival-Selection}
\end{algorithm}
The \nsgaIII algorithm, originated in~\cite{DebJain2014} is shown in Algorithm~\ref{alg:nsga-iii}. Initially, a population of size $\mu$ is created by choosing $\mu$ individuals from $\{0,1\}^n$ uniformly at random. Then in each iteration $t$, a multiset $Q_t$ of $\mu$ new offspring is created by $\mu$ times choosing an individual $s \in P_t$ uniformly at random and applying \emph{standard bit mutation} on $s$, i.e. each bit is flipped independently with probability $1/n$. During the survival selection, the parent and offspring populations $P_t$ and $Q_t$ are merged into $R_t$. Then $R_t$ is partitioned into layers $F^1_{t+1},F^2_{t+1},\dots$ using the \emph{non-dominated sorting algorithm}~\cite{Deb2002} where $F^1_{t+1}$ consists of all non-dominated individuals, and $F^i_{t+1}$ for $i>1$ of individuals only dominated by those from $F^1_{t+1},\dots,F^{i-1}_{t+1}$. Then the critical rank $i^*$ with $\sum_{i=1}^{i^*-1} \lvert{F_t^i}\rvert < \mu$ and $\sum_{i=1}^{i^*} \lvert{F_t^i}\rvert \geq \mu$ is determined. 
All individuals with a lower rank than $i^*$ are included in $P_{t+1}$, while the remaining individuals are selected from $F_t^{i^*}$ using Algorithm~\ref{alg:Survival-Selection}. Hereby, a normalized objective function $f^n$ is computed and then each individual with rank at most $i^*$ is associated with reference points. For the first, we use the normalization procedure from~\cite{WiethegerD23} which can be also used for maximization problems as shown in~\cite{OprisNSGAIII}. We omit detailed explanations as they are not needed for our purposes. For an $m$-objective function $f\colon \{0,1\}^n\rightarrow \mathbb{N}_0^m$, the normalized fitness vector $f^{n}(x):=(f_1^n(x),\dots,f_m^n(x))$ of a search point $x$ is computed as
\begin{align*}
f_j^n(x)
    =\frac{f_j(x)-y_j^{\min}}{y_j^{\text{nad}}-y_j^{\min}}
\end{align*}
for each $j \in [m]$ where $y^{\text{nad}}:=(y_1^{\text{nad}}, \ldots, y_m^{\text{nad}})$ and 
$y^{\min}:=(y_1^{\min}, \dots, y_m^{\min})$ from the objective space are called
\emph{nadir} and \emph{ideal} points, respectively. Computing the nadir point is not trivial and we have $y_j^{\text{nad}} \geq \varepsilon_{\text{nad}}$, and $y_j^{\text{min}} \leq y_j^{\text{nad}} \leq y_j^{\text{max}}$ for every $j \in [m]$ where $\varepsilon_{\text{nad}}$ is a positive threshold set by the user (see~\cite{Blank2019} or~\cite{WiethegerD23} for the details). Further, $y_j^{\max}$ and $y_j^{\min}$ are the maximum and minimum value in objective $j$ from all search points seen so far (i.e. from $P_0,Q_0,\ldots , P_t,Q_t$). 
After computing the normalisation, each individual $x$ is associated with the reference point $\text{rp}(x)$ such that the distance between $f^n(x)$ and the line through the origin and $\text{rp}(x)$ is minimal. 
We use the same set of reference points $\refer$ as proposed in~\cite{DebJain2014}, originated in~\cite{Das1998}. The points are defined as
$$\mathcal{R}_p=\left\{\left(\frac{a_1}{p}, \ldots ,\frac{a_m}{p} \right) 
    \text{ } \Big| \text{ }  
    (a_1,\dots,a_m) \in \mathbb{N}_0^m, 
    \sum_{i=1}^m a_i = p
\right\}$$

where $p \in \mathbb{N}$ is a parameter one can choose according 
to the fitness function $f$. These are uniformly distributed on the simplex determined by the unit vectors $(1,0,\dots,0)^{\intercal},(0,1,\dots,0)^{\intercal},\dots,(0,0,\dots,1)^{\intercal}$.

Then, one iterates through all the reference points where the reference point with the fewest associated individuals that are already selected for the next generation $P_{t+1}$ is chosen. A reference point is omitted if it only has associated individuals that are already selected for $P_{t+1}$ and ties are broken uniformly at random. Next, from the individuals associated to that reference point who have not yet been selected, the one closest to the chosen reference point is selected for the next generation, where ties are again broken uniformly at random. Once the required number of individuals is reached (i.e. if $\lvert{Y_t}\rvert+\lvert{\tilde{F}_t^{i^*}}\rvert = \mu$) the selection ends.

For our analysis, we need the following key lemma regarding the cover number of a Pareto-optimal fitness vector. This also includes the fact that NSGA-III protects Pareto-optimal solutions meaning that if the population size $\mu$ is larger than a set of mutually incomparable solutions and a Pareto-optimal fitness vector is covered, then it is covered for all future generations. It is a combination of Lemma~2 from~\cite{OprisNSGAIII} and Lemma~3.4 from~\cite{opris2025multimodal}. 

\begin{lemma}
\label{lem:cover-number-NSGAIII}
Consider NSGA-III on an $m$-objective function $f$ with $\varepsilon_{\text{nad}} \geq f_{\max}$ and a set $\refer$ of reference points for $p \in \mathbb{N}$ with $p \geq 2m^{3/2}f_{\max}$. Denote by $P_t$ the current population and by $V$ the Pareto front of $f$. Let $S$ be a maximum set of mutually incomparable solutions and suppose that $\mu \geq |S|$. Then the following properties hold.
    \begin{enumerate}
        \item[(1)] If $\mu \geq |S|$ then for each $x \in F_t^1$ there is an $y \in P_{t+1}$ weakly dominating $x$. 
        \item[(2)] Let $v \in \mathcal{P}$ and $0 \leq \alpha \leq \mu/|S|$. If $c_t(v) \geq \alpha$ then also $c_{t+1}(v) \geq \alpha$.
        \item[(3)] Let $v \in \mathcal{P}$ and suppose that $c_{t+1}(v)<c_t(v)$. Then $c_{t+1}(w) \leq c_t(v)$ for every $w \in \mathcal{P}$.
        \item[(4)] Suppose that every $x \in P_t$ is Pareto-optimal. Then $d_t:=\max\{c_t(v) \mid v \in \mathcal{P}\}$ does not increase.
    \end{enumerate}
\end{lemma}

The $\mOMM$ benchmark, originated in \cite{Zheng2023Inefficiency}, is defined as follows. 

\begin{definition}
Let $m$ be divisible by $2$ and let the problem size be a multiple of $m/2$. Then the $m$-objective function \mOMM is defined by
$\mOMM: \{0,1\}^n \to \mathbb{N}_0^m$ as 
\[
\mOMM(x) = (f_1(x), \ldots ,f_m(x))
\]
with 
\[
f_\ell(x)=
\begin{cases}
    \sum_{i=1}^{2n/m} x_{i+n(\ell-1)/m}, & \text{ if $\ell$ is odd,} \\
    \sum_{i=1}^{2n/m} (1-x_{i+n(\ell-2)/m}), & \text{ else,}
\end{cases}
\]
for all $x=(x_1, \ldots ,x_n) \in \{0,1\}^n$.
\end{definition}

In $\mOMM$, the bit string is divided into $m/2$ blocks, where $f_{2j-1}$ and $f_{2j}$ correspond to block $j$. Specifically, $f_{2j-1}$ counts the number of ones, and $f_{2j}$ counts the number of zeros in block $j$. Every search point is Pareto-optimal, as the total sum of objectives of any bit string is $n$. A Pareto-optimal set thus, which is also a maximum set of mutually incomparable solutions, has cardinality $(2n/m + 1)^{m/2}$, since for each block $j \in [m/2]$ there are at most $2n/m+1$ many fitness values $(f_{2j-1},f_{2j})$.

\section{Population Dynamics of NSGA-III on $2$-OMM}

\textbf{Bounding the Maximum Cover Number:} First, we establish a general upper bound on the time required to cover a subset of the Pareto front of a given cardinality with high probability. Then, we provide an additional bound on the time needed to evenly distribute solutions across that subset. 

\begin{lemma}
\label{lem:cover-number-general-1}
    Consider \nsgaIII on $f \coloneqq 2$-\OMM under the same conditions as in Lemma~\ref{lem:cover-number-NSGAIII}. (i.e. $p \geq 4\sqrt{2}n$ since $f_{\max} = n$ and $\mu \geq n+1$). Then for a given natural $\alpha \leq 3n/8$ there is $\mathcal{A} \subset P$ with cardinality $\alpha$ which is covered in $64\alpha$ generations with probability at least $1-e^{-\Omega(\alpha)}$.
\end{lemma}

\begin{proof}
    Denote by $F$ the Pareto front of $2$-OMM. By a classical Chernoff bound the probability is at least $1-e^{-\Omega(n)}$ that there is an individual $x$ initialized with $f_j(x) \in [3n/8,5n/8]$. Suppose that this happens and fix such an individual $x_0$. Let $\mathcal{A} \coloneqq \{v \in F \mid v_i \in [f_i(x_0)-\alpha/2, f_i(x_0) + \alpha/2] \text{ for all }i \in \{1,2\}\} \subset [3n/8- \alpha/2,5n/8+\alpha/2]$. 
    Then we see that $|\mathcal{A}| \geq \alpha$. Fix a covered $v \in \mathcal{A}$ and another uncovered $w \in \mathcal{A}$. We show with probability at least $1-e^{-\Omega(\alpha)}$ the vector $w$ is covered after $64\alpha$ generations. Let $B_t \coloneqq \{x\in P_t \mid f(x) \in \mathcal{A}\}$ and $d_t \coloneqq \min_{x \in B_t} |f_1(x)-w_1|$. Note that $0 \leq d_t \leq \alpha$. Further, we see that $w$ is covered if $d_t=0$ since if $f_1(x)=w_1$ then also $f_2(x)=w_2$ due to $f_1(x)+f_2(x) = w_1+w_2 = n$. By Lemma~\ref{lem:cover-number-NSGAIII}(1), $d_t$ cannot increase, but it can be decreased in one single trial by choosing $x \in P_t$ with $|f_1(x)-w_1| = d_t$ as parent (probability at least $1/\mu$) and then flipping a one bit to zero and not changing any other bit if $f_1(x)-w_1>0$. On the other hand, if $f_1(x)-w_1<0$, flip a zero bit to one. Both happen with probability at least $(3n/8- \alpha/2)/n \cdot (1-1/n)^{n-1} \geq (3n/8 - 3n/16)/(en) \geq 3/(16e)$. Then in one generation, $d_t$ decreases with probability at least
    $1-(1-\frac{3}{16e \mu})^\mu \geq \frac{3/(16e)}{1+3/(16e)} \geq \frac{3}{32e} \eqqcolon p$ where the first inequality is due to Lemma~10 in~\cite{Badkobeh2015}. Note that $p = \Omega(1)$ and for each $i \in [\alpha]$, define the random variable $X_i$ as the number of generations such that $d_t = i$. Then the time until $d_t=0$ is $X \coloneqq \sum_{i=1}^\alpha X_i$, the latter stochastically dominated by the independent sum $Y \coloneqq \sum_{i=1}^\alpha Y_i$ of geometrically distributed random variables with success probability $p_i=p = \Omega(1)$. Note that $\E{Y} = \alpha/p$ and we obtain by Theorem~1 in~\cite{Witt14} for $s \coloneqq \sum_{i=1}^{\alpha} 1/p_i^2 = 1024 \alpha e^2/9$ and $\lambda \geq 0$ the inequality $
    \Pr(Y \geq \E{Y} + \lambda) \leq \exp(-\frac{1}{4} \min\{\frac{\lambda^2}{s}, \lambda p \})$.
    For $\lambda = \alpha/p$ we obtain $\Pr(X \geq 64\alpha) \leq \Pr(X \geq 64e\alpha/3) = \Pr(X \geq 2\alpha/p) \leq \Pr(Y \geq 2\alpha/p) = \Pr(Y \geq \expect{Y} + \alpha/p) \leq e^{-\Omega(\alpha)}$. By a union bound on all $w \in \mathcal{A}$ we see that $\mathcal{A}$ is covered in $64\alpha$ generations with probability $1-e^{-\Omega(\alpha)}$.
\end{proof}

Now we give the following upper bound on the time such that, with high probability, the solutions are evenly spread on a set $\mathcal{A}$ with cardinality $\alpha$ or, in other words, the cover number of each $v \in \mathcal{A}$ is bounded by $\lceil{\mu/\alpha}\rceil$ from above.

\begin{lemma}
\label{lem:cover-number-general-2}
    Consider \nsgaIII on $f \coloneqq 2$-\OMM and suppose that all conditions of Lemma~\ref{lem:cover-number-NSGAIII} are satisfied. Suppose that $\mu = \text{poly}(n)$. Let $\alpha \leq 3n/8$ be a natural number and let $\gamma:=\min\{\lceil{n/\ln(n)}\rceil,\lceil{\mu/\alpha}\rceil\}$. Then after $84 \alpha + 46\gamma$ generations, each $v \in \N_0^m$ has cover number at most $\lceil{\mu/\alpha}\rceil$ with probability $1-o(1)$. Hence, if $\alpha \geq n/\ln(n)$ then $\gamma \leq \alpha$ and therefore, $130\alpha$ generations suffice. The expected number of generations is $O(\alpha + \gamma)$.  
\end{lemma}

\begin{proof}
    Denote by $F$ the Pareto front of $2$-OMM. By Lemma~\ref{lem:cover-number-general-1} there is a set $\mathcal{A} \subset P$ with cardinality $\alpha$ which is covered after $64\alpha$ generations with probability at least $1-e^{-\Omega(\alpha)}$. Suppose that this happens. Denote the decrease of the cover number of a vector $v \in F$ before reaching $\lceil{\mu/\alpha}\rceil$ as a \emph{success}. If a success occurs, we see that the cover number of all other Pareto-optimal vectors is at most $\lceil{\mu/\alpha}\rceil$ by Lemma~\ref{lem:cover-number-NSGAIII}(3) and it cannot increase by Lemma~\ref{lem:cover-number-NSGAIII}(4) (since every solution is Pareto-optimal) and hence, the lemma holds. We show with probability $1-o(1)$ that all $v \in \mathcal{A}$ have a cover number of at least $\lceil{\mu/\alpha}\rceil$ or a success occurred after further $46\gamma + 20\alpha$ generations. In the former case we have that $\lceil{\mu/\alpha}\rceil = \mu/\alpha$ and all other $v \in F \setminus \mathcal{A}$ have cover number $0$ and hence, the cover number of all $v \in F$ is also bounded by $\gamma$. Depending on the value of $\gamma$, we consider two cases where $F$ denotes the Pareto front of $2$-OMM.
    
    \textbf{Case 1:} Let $\lceil{\mu/\alpha}\rceil \leq \lceil{n/\ln(n)}\rceil$ (i.e. $\gamma = \lceil{\mu/\alpha}\rceil$). Fix $v \in F$, denote by $c_t$ its cover number and for $j \in [\gamma-1]$ let $X_j$ be a random variable that counts the number of generations with $c_t=j$. Then the number of generations until a success occurs or the cover number of $v$ is at least $\gamma$ is at most $X \coloneqq \sum_{j=1}^{\gamma-1} X_j$. Note that $c_t$ can be increased by choosing an individual $y$ with $f(x)=f(y)$ as parent and flipping no bits (prob. $1/\mu \cdot (1-1/n)^n \geq 1/(4 \mu)\eqqcolon\sigma_t$). Hence, the probability of increasing $c_t$ in one generation is at least 
    $1-(1-\sigma_t)^{\mu} \geq \frac{\sigma_t \mu}{1+\sigma_t \mu} = \frac{1/4}{1+1/4} = \frac{1}{5}$.
    Hence, $X$ is stochastically dominated by an independent sum $Z \coloneqq \sum_{j=1}^{\gamma-1}Z_j$ of geometrically distributed random variables $Z_j$ with parameter $p=1/5$. 
    Then $\E{X} \leq \E{Z} \leq 5 \gamma$ and hence, by Theorem~1 in~\cite{Witt14}, we obtain for $s \coloneqq \sum_{i=1}^{\gamma-1} 1/p_i^2 \leq  25\gamma$, and $\lambda \geq 0$ the inequality $\Pr(Z \geq \E{Z} + \lambda) \leq \exp(-\frac{1}{4} \min\{\frac{\lambda^2}{s}, \lambda p \})$
    and for $\lambda=40 \gamma + 20\alpha$ we obtain $\Pr(X \geq (5 + 40) \gamma + 20\alpha) = \Pr(X \geq 5 \gamma + (40 \gamma + 20\alpha)) \leq  \Pr(Z \geq \E{Z} + 40 \gamma + 20\alpha) \leq e^{-2\gamma - \alpha}$. By a union bound on $|\mathcal{A}| = \alpha$ different Pareto-optimal vectors, we see that with probability at most $\alpha \cdot e^{-2 \gamma - \alpha} = o(1)$ a success occurred or the cover number of all $v \in \mathcal{A}$ is at least $\lceil{\mu/\alpha}\rceil$ after further $45 \gamma + 20\alpha$ generations.

    \textbf{Case 2:} Suppose that $\lceil{\mu/\alpha}\rceil > \lceil{n/\ln(n)}\rceil$ (i.e. $\gamma=\lceil{n/\ln(n)}\rceil$). 
    Fix $v \in F$, and let $Y_t$ denote the number of individuals $x$ such that $f(x) = v$ at generation $t$. Then with probability $1-e^{-2\gamma - \alpha}$ after further $45\gamma + 20\alpha$ generations by Case~1, a success occurred or the cover number of $v$ is at least $n\ln(n)$. Suppose the latter (otherwise the statement of the lemma holds) and let $Z_t$ be the number of newly created individuals with fitness vector $v$ in generation $t$. We have $\expect{Z_t} \geq Y_t/4 \geq n/(4 \ln(n))$: A generation consists of $\mu$ independent trials and in each trial, with probability at least $n/(\ln(n)\mu)$, an individual $x$ with $f(x)=v$ is selected as the parent, and during mutation, no bit is flipped with probability at least $(1 - 1/n)^n \geq 1/4$. Hence, by a classical Chernoff bound, $\Pr(Z_t \leq 3/5 \cdot \expect{Z_t}) = \Pr(Z_t \leq (1-2/5) \cdot \expect{Z_t}) \leq e^{-\Omega(\expect{Z_t})} = e^{-\Omega(n/\ln(n))}$. Hence, with probability at least $1 - e^{-\Omega(n/\ln(n))}$, we have $Y_{t+1} \geq \min\left\{Y_t + 3/5 \cdot \expect{Z_t}, \lceil{\mu/\alpha}\rceil\right\} = \min\left\{Y_t + 3Y_t/20, \lceil{\mu/\alpha}\rceil\right\} = \min\left\{23Y_t/20, \lceil{\mu/\alpha}\rceil\right\}$. In other words, $Y_t$ increases by a factor of at least $23/20$ unless the value $\lceil \mu/\alpha \rceil$ has already been reached. Note that $\lceil{n/\ln(n)}\rceil$ such generations in succession are sufficient to reach a cover number of $v$ of at least $\lceil \mu/\alpha \rceil$ or a success occured, since $n/\ln(n) \cdot (23/20)^{n/\ln(n)} = \omega(\mu)$. Moreover, such a sequence of generations occurs with probability at least $1 - e^{-\Omega(n/\ln(n))}$ by a union bound on all these generations. By a further union bound on all $v \in \mathcal{A}$, we see that after $\lceil{n/\ln(n)}\rceil = \gamma$ generations a success occurred or the cover number of each $v \in F$ is at least $\lceil{\mu/\alpha}\rceil$ with probability $1-o(1)$. 
    
    Hence, in any case we see that with probability $1-o(1)$, after $\kappa:=64\alpha + 45\gamma + 20\alpha + \gamma = 84\alpha + 46\gamma$ generations, each $v \in F$ has cover number at most $\lceil{\mu/\alpha}\rceil$ with probability $1-o(1)$. If this does not happen, we repeat the arguments from either Case~1 or Case~2 for another period of $\kappa$ generations, including the preceding phase from Lemma~\ref{lem:cover-number-general-1} to cover $\mathcal{A}$ if necessary. The expected number of periods is $1+o(1)$, concluding the proof.
    \end{proof}

\textbf{Controlling the Exploration of Search Points}: First, we bound the spread of solutions in $O(n/\ln(n))$ generations.

\begin{lemma}
\label{lem:spread-of-solutions-1}
Consider \nsgaIII on $2$-OMM under the same conditions as in Lemma~\ref{lem:cover-number-NSGAIII}. Suppose that $\mu = \text{poly}(n)$ and let $c>0$ be a constant. Then, after $c n/\ln(n)$ generations, there is no $y \in P_t$ with $\ones{y} \geq 3n/4$ with probability $1-o(1)$.  
\end{lemma}

\begin{proof}
    Let $d_t \coloneqq \min\{\max\{3n/4-\ones{y},0\} \mid y \in P_t\}$. By a classical Chernoff bound each individual $x$ satisfies $3n/8 < \ones{x} < 5n/8$ with probability $1-\mu e^{-\Omega(n)} = 1-o(1)$ after initialization. Suppose that this happens.   
    Then $d_0 \geq n/8$ and therefore, in order to create an individual $y$ with $\ones{y} \geq 3n/4$ within $\lceil{cn/\ln(n)}\rceil \leq 2cn/\ln(n)$ generations, it is necessary that $d_t$ reaches $0$, particularly decreases by at least $\ln(n)/(16c)$ in one such iteration. This requires that at least $\ell:=\lceil{\ln(n)/(16c)}\rceil$ many zero bits are flipped simultaneously in one individual. The latter happens with probability at most $\binom{n}{\ell} \left(\frac{1}{n}\right)^{\ell} = \frac{n!}{\ell!(n-\ell)!n^\ell} \leq \frac{1}{\ell!} \leq \frac{e^\ell}{\ell^\ell} = e^{-\omega(\ell)} = e^{-\omega(\ln(n))}$ in one single trial where the last inequality is due to Stirling's formula. By a union bound on at most $\mu \lceil{c n/\ln(n)}\rceil$ mutation steps after $cn/\ln(n)$ generations, we see that the probability is $o(1)$ to decrease $d_t$ by at least $\ell$ one time within $\lceil{c n/\ln(n)}\rceil$ generations (since $\mu = \text{poly}(n)$), concluding the proof.
\end{proof}

We now bound the exploration of search points on an interval of the form $[n - n^b, n - n^a]$ for constants $0 \leq a < b \leq 3/4$ towards the all-one string by providing a lower bound on the number of generations required to traverse this interval.

   \begin{lemma}
    \label{lem:spread-of-solutions-2}
    Consider \nsgaIII on $2$-OMM under the same conditions as in Lemma~\ref{lem:cover-number-NSGAIII}. Let $0 \leq a<b \leq 3/4$ be two constants. Assume that the maximum cover number is at most $\beta=o(n^{1-b})$. Suppose every $x \in P_t$ satisfies $\ones{x} \leq n-n^b$. Then with probability $1-o(1)$ NSGA-III requires more than $(b-a)n\ln(n)/(32e\beta)$ generations to create an individual $x$ with $\ones{x} \geq n-n^a$. Hence, the expected number of generations is at least $\Omega(n\ln(n)/\beta)$.
    \end{lemma}

    \begin{proof}
       Consider $Y_t \coloneqq \max\{\max\{\ones{x},n-\lceil{n^b}\rceil\} \mid x \in P_t\}$. If all individuals $x$ satisfy $\ones{x} \leq n-\lceil{n^b}\rceil$ (which is the case at the beginning) then $Y_t = n-\lceil{n^b}\rceil$ and we created an $x$ with $\ones{x} \geq n-n^a$ if $Y_t \geq n - \lfloor{n^a}\rfloor$. At first we bound the probability $p^*$ to increase $Y_t$ by at least $8$ in one generation from above as follows. In one single trial for each $i \in \{0, \ldots , Y_t\}$ one can choose an $x \in P_t$ with $\ones{x}=Y_t-i$ (i.e. $\zeros{x}=n-Y_t+i$) (prob. at most $\beta/\mu$) and then flip $i+8$ zero bits (prob. at most $\binom{n-Y_t+i}{i+8} \cdot 1/n^{i+8}$). By a union bound on all $i$, we obtain that $Y_t$ increases by at least $8$ in a single trial with probability at most
        $\frac{\beta}{\mu} \sum_{i=0}^{Y_t} \binom{n-Y_t+i}{i+8} \frac{1}{n^{i+8}} \leq \frac{\beta}{\mu} \sum_{i=0}^{Y_t} \binom{\lceil{n^b}\rceil+i}{i+8} \frac{1}{n^{i+8}}
        \leq \frac{\beta}{\mu} \left(\frac{\lceil{n^b}\rceil}{n}\right)^8 \sum_{i=0}^{Y_t} \frac{(\lceil{n^b}\rceil+i) \cdot \ldots \cdot (\lceil{n^b}\rceil+1)}{n^i(i+8)!}
        \leq \frac{\beta}{\mu} \left(\frac{\lceil{n^b}\rceil}{n}\right)^8$(where we used $\lceil{n^b}\rceil+i \leq 2n$ and $\sum_{i=0}^{Y_t} \frac{2^i}{(i+8)!} \leq 1$). Hence, by a union bound on $\mu$ single trials, we obtain the inequality $p^* \leq \beta \cdot (\lceil{n^b}\rceil/n)^8 \leq 256\beta/n^2$ (since $b \leq 3/4$ as well as $\lceil{r}\rceil \leq 2r$ for all $r \geq 1$ is satisfied). Again by a union bound, $Y_t$ changed by at least $8$ after $(b-a)n \ln(n)/(32a\beta)$ generations with probability $o(1)$. So we assume that $Y_t$ is never changed by at least $8$ and for $Y_t \in [n-\lceil{n^b}\rceil, \ldots , n - \lfloor{n^a}\rfloor]$ and natural $1 \leq \ell \leq (\lceil{n^b}\rceil-\lfloor{n^a}\rfloor)/8$ let $X_\ell$ be the random variable which counts the number of generations with $Y_t \in \{n-\lceil{n^b}\rceil+8(\ell-1), \ldots , n-\lceil{n^b}\rceil+8\ell-1\}$. Now, for $k:=k(n,\ell):=\lceil{n^b}\rceil - 8(\ell-1)$ we justify that $X_\ell$ stochastically dominates a geometrically distributed random variable $Z_\ell$ with success probability $p_\ell = 1.5e\beta k/n$:
        
        A necessary condition that $Y_t$ leaves $\{n-k, \ldots , n-\lceil{n^b}\rceil+8\ell-1\}$ is that $Y_t$ increases by one in a generation which happens with probability at most $
        \frac{\beta}{\mu} \sum_{i=0}^{Y_t} \binom{n-Y_t+i}{i+1} \frac{1}{n^{i+1}} \leq \frac{\beta}{\mu} \sum_{i=0}^{Y_t} \binom{k+i}{i+1} \frac{1}{n^{i+1}} \leq \frac{\beta}{\mu} \frac{k}{n} \sum_{i=0}^{Y_t} \frac{(k+i) \cdot \ldots \cdot (k+1)}{n^i(i+1)!} \leq \frac{1.5e\beta k}{\mu n}$ (by choosing a parent $x$ with $\ones{x} = Y_t-i$ and then flipping $i+1$ zero bits for $i \in \{0, \ldots , Y_t\}$).
        For the last inequality we used $k+i \le (1+\ln(1.5))n$ for all $i \in \{0, \ldots , Y_t\}$ and $n$ sufficiently large, and $\sum_{i=0}^{Y_t} \frac{(1+\ln(1.5))^i}{(i+1)!} \leq \sum_{i=0}^{\infty} \frac{(1+\ln(1.5))^i}{(i+1)!} = e^{1+\ln(1.5)} = 1.5e$.
        Then apply a union bound on $\mu$ trials to finish the justification.
        
        This implies that for $\delta:=\delta(a,b,n):=\lfloor{\frac{\lceil{n^b}\rceil-\lfloor{n^a}\rfloor}{8}}\rfloor$ the number $T$ of generations until $Y_t \geq n-\lfloor{n^a}\rfloor$ (which is at least $\sum_{\ell=1}^\delta X_\ell$) stochastically dominates the independent sum $Z \coloneqq \sum_{\ell=1}^{\delta} Z_\ell$ of geometrically distributed random variables $Z_\ell$. Note also that $\ln(n) \leq \sum_{i=1}^n 1/i \leq \ln(n)+1$ and therefore $\sum_{i=1}^n 1/i - \sum_{i=1}^q 1/i \geq \ln(n)-(\ln(q)+1) = \ln(n/q)-1$ for $q \in [n]$. Therefore, we obtain
        $\E{Z} = \sum_{\ell=1}^\delta \E{Z_\ell} = \sum_{\ell=1}^\delta \frac{1}{p_\ell} =
        \sum_{\ell=1}^{\delta} \frac{n}{1.5e\beta(\lceil{n^b}\rceil - 8(\ell-1))}
        =\frac{n}{12e\beta} \sum_{\ell=0}^{\delta-1} \frac{1}{\lceil{n^b}\rceil/8 - \ell} \geq \frac{n}{12e\beta} \sum_{\ell=0}^{\delta-1} \frac{1}{\gamma - \ell}
        \geq \frac{n}{12e\beta} \left(\sum_{\ell=0}^{\gamma-1} \frac{1}{\gamma-\ell}-\sum_{\ell=\delta}^{\gamma-1} \frac{1}{\gamma - \ell} \right)
         \geq \frac{n}{12e\beta} \left(\sum_{\ell=1}^{\gamma} \frac{1}{\ell}-\sum_{\ell=1}^{\gamma-\delta} \frac{1}{\ell} \right)
        \geq \frac{n}{12e\beta} \left(\ln\left(\frac{\gamma}{\gamma-\delta}\right) -1\right)
        \geq \frac{n}{12e\beta} \left(\ln\left(\frac{\lceil{n^b}\rceil/8}{\lfloor{n^a}\rfloor/8+1}\right) -1\right)
        \geq \frac{n}{12e\beta} \left(\ln \left(\frac{n^b}{n^a+8}\right) - 1\right)
        \geq \frac{(b-a)n \ln(n)}{16e\beta}$ for $n$ sufficiently large. Then, under the condition that $Y_t$ never changes by at least $8$ within $(b-a)n \ln(n)/(32e\beta)$ generations, we see by Theorem~1 in~\cite{Witt14} that for $\lambda:=\expect{Z}/2$ and $s:= \sum_{\ell=1}^{\delta} 1/p_\ell^2 = \sum_{\ell=1}^{\delta} \frac{n^2}{2.25e^2\beta^2(\lceil{n^b}\rceil-8(\ell-1))^2} \leq \sum_{j=1}^{\infty} \frac{n^2}{2.25e^2\beta^2j^2} \leq \frac{n^2 \pi^2}{13.5 e^2 \beta^2}$ (due to $\sum_{i=1}^\infty 1/i^2 = \pi^2/6$) the inequality $\Pr(T \leq \frac{(b-a)n \ln(n)}{32e\beta}) \leq \Pr(Z \leq \expect{Z}/2) = \Pr(Z \leq \E{Z}-\expect{Z}/2) \leq \exp(-\lambda^2/(2s)) = o(1)$ holds. This proves the lemma with the law of total probability.
    \end{proof}

\section{A Lower Runtime Bound}

In this section, we establish the desired lower bound on the runtime of \nsgaIII on the $2$-OMM problem by putting together the results from the previous section. 

\begin{theorem}
\label{thm:lower-bound}
Consider \nsgaIII on $2$-OMM under the same conditions as in Lemma~\ref{lem:cover-number-NSGAIII}. Further suppose that $\mu \in O(\ln(n)^cn)$ for a constant $0 < c<1$. Then the expected number of generations to cover the whole Pareto front is at least $\Omega(n^2 \ln(n)/\mu)$.  
\end{theorem}

\begin{proof}
Fix a constant $\chi>0$ such that $\mu \leq \chi \ln(n)^c n$ for $n$ sufficiently large. At first we see by Lemma~\ref{lem:cover-number-general-1} that with probability $1-o(1)$ there is no individual $y$ in $P_t$ with $\ones{y} \geq 3n/4$ within $130\lfloor{n/\ln(n)}\rfloor$ generations.
Further, by Lemma~\ref{lem:cover-number-general-2} on $\alpha = \lfloor{n/\ln(n)}\rfloor$ we obtain that after $130\alpha$ generations the maximum cover number is at most $\lceil{\mu / \alpha}\rceil \leq \frac{\mu}{n/\ln(n)-1} + 1 \leq 2\mu \ln(n)/n \leq 2\ln(n)^{1+c}$ with probability $1-o(1)$. Suppose that this happens. We now apply Lemma~\ref{lem:spread-of-solutions-2} with $b=3/4$ and $a=1/2$ to obtain with probability $1-o(1)$ that after further $(b-a) n\ln(n)/(32 e \cdot 2 \ln(n)^{1+c}) = n/(256 e \ln(n)^c) = d_0n/\ln(n)^c$ generations for $d_0=1/(256 e)$, no solution $x$ with $\ones{x} \geq n-n^{1/2}$ is created. If this happens, apply Lemma~\ref{lem:cover-number-general-2} on that number of generations for $\alpha = \lfloor{d_0n/(130\ln(n)^c)}\rfloor$ (note that $130 \alpha \leq d_0 n/\ln(n)^c$) to obtain for $e_0:=260\chi/d_0$ that with probability $1-o(1)$ the maximum cover number is at most $\lceil{\mu/\alpha}\rceil \leq \lceil{\chi \ln(n)^c n/\alpha}\rceil \leq e_0 \ln(n)^{2c} = \max\{e_0\ln(n)^{2c}, 16\mu/(3n)\}$ for $n$ sufficiently large (the latter equality holds due to $e_0 \ln(n)^{2c} \in \omega(\mu/n)$).

Suppose that these two happen. In the following, we iteratively reduce the maximum cover number as the population approaches the extreme solution $1^n$. To this end, let $\ell:=\lceil{(2c+1)/(1-c)}\rceil \in O(1)$ and suppose that for $j \in \{0, \ldots , \ell-1\}$ there are constants $0 < b_j < 1/2$ and $d_j,e_j \geq 0$ such that after $d_jn\ln(n)^j/\ln(n)^{(2+j-1)c}$ generations, no solution $x$ with $\ones{x} \geq n-n^{b_j}$ is created, that 
$(\ln(n))^{(2+j)c - j} = \omega(\mu/n)$ and the maximum cover number is at most $\beta = e_j (\ln(n))^{(2+j)c - j} = \max\{e_j (\ln(n))^{(2+j)c - j},16\mu/(3n)\}$ (where the case $j = 0$ already occurred). Now fix a further constant $b_{j+1}$ with $1/8 < b_{j+1} < b_j$. Then again by Lemma~\ref{lem:spread-of-solutions-2} we see that with probability $1-o(1)$ in $
\frac{(b_j-b_{j+1}) n \ln(n)}{32e \cdot \beta} = \frac{(b_j-b_{j+1}) n \ln(n)}{32e \cdot e_j (\ln(n))^{(2+j)c - j}} = \frac{d_{j+1} n \ln(n)^{j+1}}{\ln(n)^{(2+j)c}}$ generations for $d_{j+1}=\frac{b_j-b_{j+1}}{32e \cdot e_j}$ no solution $x$ with $\ones{x} \geq n-n^{b_{j+1}}$ is created. After this time, by Lemma~\ref{lem:cover-number-general-2} on $\alpha = \min\{\lfloor{\frac{d_{j+1}n\ln(n)^{j+1}}{(130\ln(n)^{(2+j)c})}}\rfloor,\lfloor{3n/8}\rfloor\}$, the maximum cover number is at most $\lceil{\mu/\alpha}\rceil \leq \max\{\frac{260 \mu \ln(n)^{(2+j)c}}{d_{j+1} n \ln(n)^{j+1}} , \frac{16\mu}{3n}\} \leq \max\{\frac{260 \chi n \ln(n)^{(3+j)c}}{d_{j+1}n\ln(n)^{j+1}},\frac{16\mu}{3n}\} = \max\{\frac{e_{j+1}\ln(n)^{(3+j)c}}{\ln(n)^{j+1}},\frac{16\mu}{3n}\}$ for $e_{j+1}:=260\chi/d_{j+1}$ with probability $1-o(1)$. If $\ln(n)^{(3+j)c-(j+1)} = \omega(\mu/n)$, we increase $j$ by one and repeat this argument. We stop when $\ln(n)^{(3+j)c-(j+1)} = O(\mu/n)$. Since $(3+\ell)c-(\ell+1) \leq 0$, we have at most $\ell= O(1)$ such repetitions. After the last repetition we have that $\alpha = \Omega(n)$. Hence, by applying a union bound on all repetitions, we conclude that with probability $1-o(1)$, there exists a generation $t^{\text{spread}}$ such that no individual $x$ with $\ones{x} \leq n-n^{1/8}$ is created and the maximum cover number is at most $e_{\text{spread}}\mu/n$ for a constant $e_{\text{spread}}>0$. Suppose this event occurs, and apply Lemma~\ref{lem:spread-of-solutions-2} once more with $b=1/8$ and $a=1/16$. This yields that, after $\Omega(n \ln(n)/(e_{\text{spread}}\mu/n)) = \Omega(n^2 \ln(n)/\mu)$ generations in expectation (from time $t^{\text{spread}}$ onward), a search point $x$ with $\ones{x} \geq n-n^{1/16}$ is created, concluding the proof.
\end{proof}

\section{An Improved Upper Runtime Bound}

To complement our analysis, we establish an improved upper bound on the expected runtime of NSGA-III on $m$-OMM for a constant number of objectives $m$. Our approach closely follows the methodology provided by~\cite{OprisNSGAIII}, with the added consideration of the cover number.

\begin{theorem}
\label{thm:Runtime-Analysis-NSGA-III-mOMM}
Consider \nsgaIII on $\mOMM$ for a constant number $m$ of objectives under the same conditions as in Lemma~\ref{lem:cover-number-NSGAIII} with population size $\mu \geq (2n/m+1)^{m/2} =:S_m$. 
Then a Pareto-optimal set of $\mOMM$ is found in expected $O(\min\{S_m n\ln n/\mu + n\mu/S_m, n \ln(n)\})$ generations or, in other words, in expected $O(\min\{S_m n\ln n + n\mu^2/S_m, \mu n \ln(n)\})$ fitness evaluations.
\end{theorem}

\begin{proof}
    We can assume that $\mu \leq \ln(n)S_m$ and $\mu \in \omega(S_m)$ since otherwise the bound from Theorem~5.2 in~\cite{OprisNSGAIII} of $O(n \ln(n))$ expected generations holds. Fix a vector $v$ on the Pareto front. We estimate the probability not to cover $v$ after $6S_men \ln(n)/\mu +10n \lceil{\mu/S_m}\rceil$ generations. For each generation $t$ let $d_t:=\min_{x \in P_t} \sum_{j=1}^{m/2}|f_{2j-1}(x)-v_{2j-1}|$. Note that $0 \leq d_t \leq n$ and that we have covered $v$ if $d_t=0$. 
    Further, by Lemma~\ref{lem:cover-number-NSGAIII}(1), $d_t$ cannot increase. Let $y \in P_t$ be with $\sum_{j=1}^{m/2}|f_{2j-1}(y)-v_{2j-1}| = d_t$. We first increase the cover number of $f(y)$ to $\lfloor{\mu / S_m} \rfloor$ (which, by Lemma~\ref{lem:cover-number-NSGAIII}(2), can only decrease if it exceeds $\lfloor{\mu / S_m \rfloor}$, and even then not below this value), and then proceed to decrease $d_t$. The latter then happens with probability at least $\lfloor{\mu / S_m \rfloor}/\mu \cdot i/n \cdot (1-1/n)^{n-1} \geq i/(2S_men)$ in one single trial and hence, with probability at least $1-(1-i/(2S_m en))^\mu \geq \frac{i\mu/(2S_m en)}{i\mu/(2S_m en)+1} =:p_i$ in one generation. Hence, the time $T$ until $d_t = 0$ is stochastically dominated by an independent sum of geometrically distributed random variables $Y_i$ , $Z_j$ ($i \in \{1, \ldots , n\}$ and $j \in \{1, \ldots , \nu\}$ for $\nu :=n \cdot \lfloor{\mu/S_m}\rfloor$) with success probability $p_i$ and $\tilde{p} = 1/5$ respectively (compare also with the proof of Lemma~4 for the latter). Let $Y:= \sum_{i=1}^n Y_i$ and $Z:=\sum_{j=1}^{\nu} Z_j$. We have $\expect{Y} = \sum_{i=1}^n 1/p_i = \sum_{i=1}^n (1+S_men/(i\mu)) \leq n + S_men (\ln(n)+1)/\mu \leq 2S_men\ln(n)/\mu$ for $n$ sufficiently large and $\expect{Z}=5n \cdot \lfloor{\mu/S_m}\rfloor$. 
    By Theorem~1 in~\cite{Witt14} we obtain for $s=\sum_{i=1}^n 1/p_i^2 =  \sum_{i=1}^n (1+2S_men/(i\mu))^2$, $p=\min_{i \in [n]} p_i = p_1 \geq \mu/(4S_m e n)$, and $\lambda = 8mS_men\ln(n)/\mu$ that 
    $\Pr(Y \geq \expect{Y} + \lambda) \leq \exp\left(-\frac{1}{4} \min\left\{\frac{\lambda^2}{s}, \lambda p \right\}\right) \leq n^{-2m}$
    since $\lambda^2/s = \Omega(\ln^2(n))$ and $\lambda p \geq 2m\ln(n)$. Further, we see for $\tilde{s}=25\nu$, and $\tilde{\lambda}=\expect{Z}$ that $\Pr(Z \geq \expect{Z} + \tilde{\lambda}) \leq \exp(-\frac{1}{4} \min\{\frac{\tilde{\lambda}^2}{\tilde{s}}, \tilde{\lambda} \tilde{p}\}) = e^{-\Omega(n)}.$
    These two inequalities imply $\Pr(T \geq \expect{Y} + \expect{Z} + \lambda + \tilde{\lambda}) \leq \Pr(Y + Z \geq \expect{Y} + \expect{Z} + \lambda + \tilde{\lambda}) \leq \Pr(Y \geq \expect{Y} + \lambda) + \Pr(Z \geq \expect{Z} + \tilde{\lambda}) \leq 2n^{-2m}.$
    By a union bound on all possible $v$, the probability that there is a fitness vector $v$ such that $P_t$ does not contain a Pareto-optimal solution $x$ with $f(x)=v$ after $\expect{Y} + \expect{Z} + \lambda + \tilde{\lambda} \leq 10S_men \ln(n)/\mu +10n \lfloor{\mu/S_m}\rfloor = O(S_mn\ln(n)/\mu+\mu n/S_m)$ generations is at most $(2n/m+1)^{m/2} \cdot 2n^{-2m} = o(1)$. 
    If this does not happen we repeat all the above arguments. Note that in expectation, $1+o(1)$ such periods are sufficient. 
    \end{proof}

    This result improves the corresponding upper bound from~\cite{OprisNSGAIII} by a factor of $\min\{S_m/\mu,\mu/(S_m \ln(n))\}$ both in terms of generations and fitness evaluations if $S_m \leq \mu = O(S_m \ln(n))$. Along with Theorem~\ref{thm:lower-bound}, we see in the case $m=2$ that for $n+1 \leq \mu \leq (n+1)\ln(n)^c$ for a constant $0 \leq c \leq 1/2$ the full Pareto front is covered in expected $\Theta(n^2 \ln(n)/\mu)$ generations, which is a tight runtime bound. 
    
\section{Conclusions}
In this paper, we analyzed the widely used NSGA-III algorithm on the simple $m$-OMM problem and established lower runtime bounds for $m = 2$, as well as improved upper runtime bounds for a constant number $m$ of objectives compared to~\cite{OprisNSGAIII}. For $m = 2$, this leads to a tight runtime bound when employing a superconstant yet carefully chosen population size $\mu$. In this setting, NSGA-III even outperforms NSGA-II, due to its ability to distribute solutions very evenly across the Pareto front. This is very surprising, since the latter is the state of the art algorithm for bi-objective problems (with around 60000 citations). Unlike previous work~\cite{opris2025multimodal}, where NSGA-III's dynamics were analyzed on $m$-OJZJ by first exploring the local optima, and then spreading the solutions evenly across the Pareto front, our analysis required a more refined investigation of the population dynamics. In particular, we bound the maximum cover number \emph{during} the exploration process toward the all-ones string in several stages, where the spread of solutions is not hindered by local optima. These insights provide a deeper understanding of the strengths and limitations of NSGA-III and may serve as a foundation for analyzing its behavior on more complex fitness landscapes. Ultimately, this understanding can aid practitioners in developing enhanced versions of the algorithm with improved performance for efficiently optimizing problems defined by diverse and rugged fitness landscapes. Future research directions may include bounding the maximum cover number on benchmark problems, where it is necessary to \emph{reach} the Pareto front at a first glance, as well as applying the insights on population dynamics to practical scheduling and graph problems.
\bibliography{aaai2026}

\end{document}